%% file: K-cap.tex
\documentclass[sigconf]{acmart}

\usepackage{paralist}

\usepackage[linesnumbered]{algorithm2e}
\usepackage{tikz-qtree}
\usepackage{pgfplots}
\pgfplotsset{compat=1.16}
\usetikzlibrary{shapes, positioning, arrows}

\SetAlCapNameFnt{\scriptsize}
\SetAlCapFnt{\scriptsize}

\usepackage{xcolor}

\usepackage{upgreek}

\DeclareMathOperator*{\argmin}{arg\,min}

\usepackage{hyperref}
\hypersetup{bookmarksnumbered,unicode,hidelinks}

\definecolor{tab20darkblue}{HTML}{4e79a7}
\definecolor{tab20darkgreen}{HTML}{59a14f}
\definecolor{tab20darkred}{HTML}{e15759}
\definecolor{tab20darkorange}{HTML}{f28e2b}
\definecolor{tab20darkturquoise}{HTML}{499894}
\definecolor{tab20darkgray}{HTML}{79706e}
\definecolor{tab20darkbrown}{HTML}{9d7660}
\definecolor{tab20darkpurple}{HTML}{b07aa1}
\definecolor{tab20darkpink}{HTML}{e377c2}
\definecolor{tab20lightgreen}{HTML}{8cd17d}
\definecolor{tab20lightblue}{HTML}{a0cbd8}
\definecolor{tab20lightorange}{HTML}{ffbb78}
\definecolor{tab20lightred}{HTML}{ff9896}
\definecolor{tab20lightpurple}{HTML}{c5b0d5}
\definecolor{tab20lightbrown}{HTML}{c49c94}
\definecolor{tab20lightpink}{HTML}{f7b6d2}
\definecolor{tab20lightgray}{HTML}{c7c7c7}
\definecolor{tab20lightturquoise}{HTML}{9edae5}
\definecolor{tab20lightyellow}{HTML}{dbdb8d}
\definecolor{tab20darkyellow}{HTML}{bcbd22}

\makeatletter

\begin{document}

\title[Explaining \texorpdfstring{$\mathcal{ELH}$}{ELH} Concept Descriptions through Counterfactual Reasoning]{Explaining \texorpdfstring{$\mathcal{ELH}$}{ELH} Concept Descriptions \texorpdfstring{\\}{} through Counterfactual Reasoning}

\author{Leonie Nora Sieger}
\affiliation{%
  \institution{Data Science Research Group}
  \streetaddress{}
  \city{Paderborn University}
  \country{Germany}}
\email{leonie.nora.sieger@uni-paderborn.de}

\author{Stefan Heindorf}
\affiliation{%
	\institution{Data Science Research Group}
	\streetaddress{}
	\city{Paderborn University}
	\country{Germany}}
\email{heindorf@uni-paderborn.de}

\author{Yasir Mahmood}
\affiliation{%
	\institution{Data Science Research Group}
	\streetaddress{}
	\city{Paderborn University}
	\country{Germany}}
\email{yasir.mahmood@uni-paderborn.de}

\author{Lukas Blübaum}
\affiliation{%
	\institution{Data Science Research Group}
	\streetaddress{}
	\city{Paderborn University}
	\country{Germany}}
\email{lukasbl@campus.uni-paderborn.de}

\author{Axel-Cyrille Ngonga Ngomo}
\affiliation{%
	\institution{Data Science Research Group}
	\streetaddress{}
	\city{Paderborn University}
	\country{Germany}}
\email{axel.ngonga@upb.de}

\renewcommand{\shortauthors}{Sieger et al.}

\begin{abstract}
Knowledge bases are widely used for information management, enabling high-impact applications such as web search, question answering, and natural language processing. They also serve as the backbone for automatic decision systems, e.g., for medical diagnostics and credit scoring. As stakeholders affected by these decisions would like to understand their situation and verify how fair the decisions are, a number of explanation approaches have been proposed. An intrinsically transparent way to do classification is by using concepts in description logics. However, these concepts can become long and difficult to fathom for non-experts, even when verbalized.
One solution is to employ counterfactuals to answer the question, ``How must feature values be changed to obtain a different classification?'' By focusing on the minimal feature changes, the explanations are short, human-friendly, and provide a clear path of action regarding the change in prediction. While previous work investigated counterfactuals for tabular data, in this paper, we transfer the notion of counterfactuals to knowledge bases and the description logic $\mathcal{ELH}$.
Our approach starts by generating counterfactual candidates from concepts, followed by selecting the candidates requiring the fewest feature changes as counterfactuals. When multiple counterfactuals exist, we rank them based on the likeliness of their feature combinations. We evaluate our method by conducting a user survey to determine which counterfactual candidates participants prefer for explanation. 
\end{abstract}


\maketitle

\section{Introduction}
\label{sec:introduction}

Knowledge bases (KBs) are commonly used to represent information in various domains.
KBs  such as DBpedia~\cite{Auer2007DBpedia}, Wikidata~\cite{Vrandecic2014Wikidata}, or YAGO~\cite{Suchanek2007YAGO} are used in web applications, including information retrieval~\cite{Rudnik2019Searching}, information generation~\cite{Negreanu2022Rows}, web search~\cite{Blanco2015Fast} and question answering~\cite{Grau2018Corpus}. 
In the medical domain, KBs such as  \textsc{DrugBank}~\cite{Wishart2018DrugBank}, \textsc{Snomed}~\cite{Brandt2004Reasoning}, and \textsc{STRING}~\cite{Szklarczyk2019STRING} are widely used for predicting whether a molecule is safe, whether it helps against a certain disease, and what the functions of proteins are~\cite{Hu2020OGB}. 
Further applications include medical diagnostics~\cite{Chai2020Diagnosis}, credit scoring~\cite{Zhan2018loan} and hiring decisions~\cite{chung2023domain}.
Typical machine learning tasks include predicting whether the information about an entity is complete~\cite{Galarraga2017Predicting, Lajus2018Are}, correct~\cite{Chen2020Correcting}, and what category an entity falls in~\cite{Zahera2021ASSET, Gangemi2012Automatic}. 
In many of these domains, explaining algorithmic decisions of AI systems to stakeholders is important~\cite{Adadi2018Peeking,Guidotti2019Survey,Arrieta2020Explainable,Meske2022Explainable}: 
\begin{inparaenum}[(i)]
\item subjects affected by model decisions would like to understand their situation and verify fair decisions; 
\item data scientists would like to debug and improve their model; 
\item regulatory entities would like to check compliance with laws and regulations.
\end{inparaenum}

For KBs, concepts in description logics (DLs) can serve as transparent, white-box models for binary classification and many approaches for learning concepts from positive and negative examples have been proposed~\cite{Buhmann2014Inductive, 
Demir2021Drill, Kouagou2021Neural, Kouagou2021Learning, Heindorf2022EvoLearner}. For each individual in the KB, they predict whether a concept holds, i.e., whether the individual should be classified as a positive instance for the concept. However, the learned concepts can become long and complex and often exceed 20 tokens, sometimes as many as 1,000~\cite{Heindorf2022EvoLearner}. This jeopardizes their use as short, human-friendly explanations. 
Moreover, as a concept provides an explanation \emph{for each} individual (global explanation), it often contains many parts that are irrelevant to explain the prediction of a \emph{single individual} (local explanation). 
To mitigate these issues, counterfactual explanations (CEs) can serve as a form of short and actionable explanations~\cite{Molnar2022Interpretable}. 
CEs focus on an antecedent that would have caused a different outcome (classification) had it been the case~\cite{Stepin2021Survey}. 

CEs answer the question of how the classifier's input needs to be minimally changed to arrive at a different prediction~\cite{Wachter2018Counterfactual}. Dandl et al.~\cite{Dandl2020Multi-Objective} generalize this idea and take further criteria into account.

They propose a multi-objective optimization problem with four objectives:
(1)~the prediction for the CE should be as close as possible to the desired prediction;
(2)~the CE should be as similar as possible to the original instance;
(3)~feature changes should be sparse;
(4)~the CE should have plausible/likely feature values/combinations.

While CEs are widely used for tabular data~\cite{Dandl2020Multi-Objective, Stepin2021Survey, Wachter2018Counterfactual}, in this paper we transfer the notion of CEs to DLs and generate simple explanations for \emph{individuals}. As a first step in this direction, we focus on the DL $\mathcal{ELH}$ that is designed to balance expressivity and computational tractability by limiting the language constructs and supporting efficient reasoning algorithms.

Given an $\mathcal{ELH}$-KB $\mathcal K$, a concept $C$, and an individual $x$ such that $C(x)$ holds (respectively does not hold) in $\mathcal K$, we generate the counterfactual candidates $\mathcal K'$ from $\mathcal K$ in which $C(x)$ does not hold (holds) by applying changes only to the ABox. We also define criteria of non-redundancy, locality and minimality with the goal of generating practically useful CEs for ML models focusing on the individual $x$.
\input{figure_count_input}
Figure~\ref{fig:ex} shows an example with two counterfactual candidates.
In line with previous works~\cite{Wachter2018Counterfactual, Dandl2020Multi-Objective}, we define counterfactuals (CFs) as those candidates which are most similar to the original KB.
Finally, we rank these CFs according to the plausibility with which they appear in the real world, i.e., the \emph{likeliness} of the combination of their features, to construct CEs from the most plausible ones. 
We also conducted a user survey to investigate if the selected CEs are indeed preferred by users. To summarize, our contributions are as follows:
\begin{enumerate}
    \item We formally define CFs and CEs for individuals in DLs.
    \item We provide an algorithm to compute CEs for concept descriptions in $\mathcal{ELH}$.
    \item We provide a heuristic (likeliness) to choose among multiple counterfactuals such that the user will only see a small number of counterfactual explanations.
    \item We evaluate our CEs and likeliness measures via a user study.
\end{enumerate}
In what follows, Section~\ref{sec:related-work} discusses related work, Section~\ref{sec:preliminaries} introduces preliminaries,  Section~\ref{sec:approach} formalizes the notion of CEs in $\mathcal{ELH}$ and presents our algorithms to generate CEs, and Section~\ref{sec:evaluation-method} presents the user study to investigate user preferences. Finally, Section~\ref{sec:discussion} discusses our results.
For all data, code and materials needed for reproducing, and a full version including proof details, see our repository \footnote[1]{\href{https://github.com/LNSieger/Counterfactual-Explanations-DL-ELH/}{https://github.com/LNSieger/Counterfactual-Explanations-DL-ELH/ \label{repo}}}.

\section{Related Work}
\label{sec:related-work}
Current explanation approaches for description logics include justification, axiom pinpointing, as well as abduction (for both ABoxes and TBoxes).
\emph{Justification}-based approaches \cite{Kalyanpur2007Finding, Horridge11, KalyanpurPG06} aim at finding explanations of why a TBox axiom (also called general concept
inclusions or subsumption) is entailed by a knowledge base. In contrast, our work focuses on the entailment of ABox axioms.
The explanations, also called justifications, can be computed via axiom pinpointing \cite{SchlobachC03, BaaderPS07, BaaderPS07a}
, which yields a (minimal) subset of TBox axioms that entails a given TBox axiom.
\emph{Abductive} reasoning computes a hypothesis that, together with the knowledge base, is sufficient to entail a given set of axioms.
ABox abduction~\cite{PukancovaH18,Del-PintoS19,Koopmann21a} expresses both the hypothesis and observation in terms of ABox axioms. 
In $\mathcal{ELH}$, CEs for an assertion to hold are closely related to ABox abduction (see section~\ref{subsec:gen}). However, CEs for an assertion to fail can not be covered by abduction. 

Once abductive explanations or justifications have been obtained, ontology \emph{repairing}~\cite{BaaderKNP18,BaaderKKN21,BaaderKKN22} 
provides means to make minimal changes to a given knowledge base, enforcing or preventing a particular entailment. 
Moreover, approaches that employ ABox repairing~\cite{BienvenuB16} address the task of making an inconsistent ABox consistent before answering queries on the knowledge base.
While repairing approaches aim to make minimal changes to a knowledge base, our definition of optimal changes to a knowledge base is different and we take further criteria into account (e.g., our likelihood and locality measures).
Whereas \citet{Lehmann2012Framework} developed a framework for measuring the similarity of $\emph{concepts}$, our setting requires measuring the similarity of \emph{individuals} (rather than concepts).

The goal and applications of our counterfactuals is different from all the techniques named above: it is to explain the white-box machine learning model (=concept) learned by a concept learner to (lay) users. 
The difference becomes particularly apparent when we decide which counterfactual candidate is considered for counterfactual explanation.

Although counterfactuals have been mentioned in the context of description logics before, their definitions are vastly different from ours. \citet{FilandrianosTDS22} recently proposed a framework for computing counterfactual explanations for \emph{black-box} classifiers whereas we aim to explain \emph{white-box} models---namely concepts obtained from concept learners~\cite{Buhmann2014Inductive, Sarker2019Efficient, Demir2021Drill, Kouagou2021Neural, Kouagou2021Learning, Heindorf2022EvoLearner}.
\citet{IannonePF07} use the term ``counterfactual'' for negated residuals (i.e., parts of a \emph{concept}) and use them to prune the search space of a concept learner. In contrast, our counterfactuals are (individuals in) knowledge bases and we use them to explain the output of a concept learner.

\section{Preliminaries}
\label{sec:preliminaries}

We give a brief overview of the DL $\mathcal{ELH}$.
For further details, we refer the reader to existing literature~\cite{Brandt2004Reasoning, Baader2017Introduction}.

\paragraph{The Description Logic $\mathcal{ELH}$.}

\input{table-elh}
In DLs~\cite{Baader2017Introduction}, knowledge is represented by concept descriptions built from atomic concepts $A, B \in N_C$ and roles $r, s \in N_R$, where $N_C$ and $N_R$ are finite sets of concept and role names. 
Every concept name $A$ as well as the top concept $\top$ is a concept description.
As depicted in Table~\ref{tab:elh}, other $\mathcal{ELH}$-concept descriptions are constructed using the intersections ($C \sqcap D$) and existential restrictions ($\exists r.C$) from the concept descriptions $C,D$ and role $r$. 
Their semantics is defined in terms of an interpretation $\mathcal{I}=(\Delta^\mathcal{I}, \cdot^\mathcal{I})$ which consists of the non-empty set $\Delta^\mathcal{I}$, called interpretation domain, and the function $\cdot^\mathcal{I}$, called interpretation function, that assigns each $A \in N_C$ a set $A^\mathcal{I} \subseteq \Delta^\mathcal{I}$ and each $r\in N_R$ a binary relation $r^\mathcal{I} \subseteq \Delta^\mathcal{I} \times \Delta^\mathcal{I}$~\cite{Lehmann2010Concept}. 
Furthermore, each individual $x \in N_I$, where $N_I$ is a finite set of individual names, is assigned an element $x^\mathcal{I} \in \Delta^\mathcal{I}$~\cite{Lehmann2010Concept}.

Let $C, D$ be concepts, $r, s$ be role names, and $x, y\in N_I$ be individuals. 
A \emph{TBox} is a set of concept inclusion ($C \sqsubseteq D$) and role inclusion axioms ($r \sqsubseteq s$).
A concept (resp., role) assertion is an expression of the form $C(x)$ ($r(x,y)$).
An \emph{ABox} is a set of concept and role assertion axioms. 
Finally, a knowldge base $\mathcal K$ consists of a TBox and an ABox.
An interpretation $\mathcal{I}$ is a model of the KB $\mathcal{K}$ (denoted as $\mathcal I \models \mathcal K$) iff $\mathcal I$ satisfies all axioms in the TBox and ABox. 
An individual $x \in N_I$ is an instance of a concept $C$ with respect to $\mathcal{K}$, written $\mathcal{K} \models C(x)$ iff in all models $\mathcal{I}$ of $\mathcal{K}$, we have that $a^\mathcal{I} \in C^\mathcal{I}$. 
We say that $C$ holds for $x$ in $\mathcal{K}$ if $\mathcal K\models C(x)$.

In this paper, we make the unique-name assumption (UNA) and require that $x^{\mathcal I} \neq y^{\mathcal I}$ for individuals $x,y$ such that $x \neq y$ \cite{Baader2017Introduction, DLhandbook}.

\section{Counterfactuals in \texorpdfstring{$\mathcal{ELH}$}{ELH}}
\label{sec:approach}

Following Wachter et al.~\cite{Wachter2018Counterfactual} and Dandl et al.~\cite{Dandl2020Multi-Objective}, who defined CFs for black-box machine learning models with fixed-size input vectors, we transfer their definition to concept assertions in DL.
Given a KB $\mathcal K$, a concept $C\in N_C$ and an individual $x \in N_I$, then a KB $\mathcal K'$ is a counterfactual candidate for $C(x)$ iff
either $\mathcal{K} \models C(x) \text{ and } \mathcal{K}' \not\models C(x) \text{, or } \mathcal{K} \not\models C(x) \text{ and } \mathcal{K}' \models C(x)$.

That is, the evaluation of $C(x)$ differs with respect to $\mathcal K$ and $\mathcal K'$ where
$\mathcal{K}$ and $\mathcal{K'}$ are defined over the same atomic concepts and roles.

A counterfactual candidate can be seen as an intermediate step in our approach to creating CFs.
We formalize counterfactuals as a response to a user's request to change an ABox (i.e., requesting a different scenario).
Let $\mathcal{K}=(\mathcal{T}, \mathcal{A})$ be a KB, and $C(x)$ be a concept assertion for an individual $x$ in $\mathcal{K}$. 
We call the pair $P=\langle C(x),U\rangle$ a \emph{counterfactual request}, where $U\in \{\text{add}, \text{rem}\}$.
Let $\mathcal{K'} = (\mathcal{T}, \mathcal{A'})$ be a KB such that: $\mathcal{K'} \models C(x)$ if $U=\text{add}$, or $\mathcal{K'} \not \models C(x)$ if $U=\text{rem}$. 
Then we say that $\mathcal{K'}$ \emph{fulfills} the request $P$ and denote this by $\mathcal{K'} \vdash P$.
Furthermore, we call $\mathcal K'$ a counterfactual candidate for $P$ in $\mathcal{K}$. 

We assume that $\mathcal K$ does not already fulfill $P$ and hence needs to be appropriately updated. 
The fact that $\mathcal K\not\vdash P$ and $\mathcal K'\vdash P$ corresponds to the first objective (closeness to desired prediction) of CEs proposed by \citet{Dandl2020Multi-Objective}.
The creation of a counterfactual candidate can also be denoted as an update $\mathcal{K} \to \mathcal{K'}$. 
To update 
a KB  (when $U= rem$), it is often necessary to change multiple assertions from which the target assertion can be inferred. 
Due to the complex nature of TBoxes, there are sometimes multiple ways to achieve this.

\begin{example}\label{ex:update}
Let $\mathcal{K}=(\mathcal T,\mathcal A)$, where $\mathcal{T}:=\{B \sqcap C \sqsubseteq D\}$ and $\mathcal{A}:=\{B(x), C(x), D(x)\}$. 
Moreover let $P=\langle D(x),rem \rangle$ be a counterfactual request.
Then there are three updates,  $\mathcal K'_1 :=(\mathcal T, \{B(x)\}$), $\mathcal K'_2 :=(\mathcal T, \{C(x)\})$ and $\mathcal K'_3 :=(\mathcal T, \emptyset)$ that fulfill $P$.

\end{example}

Clearly, candidates \emph{without redundancy} having \emph{minimal changes} are preferred in such scenarios.
 
Intuitively, an update $\mathcal{K} \to \mathcal{K'}$ is non-redundant if no changes were made in $\mathcal K'$ so that avoiding these changes would still fulfill the counterfactual request.
In the following, $\mathcal K =(\mathcal T, \mathcal A)$ and $\mathcal K'=(\mathcal T,\mathcal A')$ denote two KBs with the same TBox ($\mathcal T$) and different ABoxes ($\mathcal A$ and 
We write $\Updelta(\mathcal A, \mathcal A')$ to denote the symmetric difference of $\mathcal A$ and $\mathcal A'$ defined as $\Updelta(\mathcal A, \mathcal A'):= (\mathcal A\setminus \mathcal A')\cup (\mathcal A'\setminus \mathcal A)$.
\begin{definition}[Non-redundancy of Changes]
\label{sec:def2}
Let $\mathcal K $ and $\mathcal K'$ be two KBs and $P=\langle C(x), U\rangle$ be a CF request. 
If $\not \exists D(y) \in \mathcal{A} \setminus \mathcal A'$ such that $(\mathcal{T}, \mathcal{A'} \cup \{D(y)\}) \vdash P$ if $U=rem$, and $\not \exists D(y) \in \mathcal{A'} \setminus \mathcal A$ so that $(\mathcal{T}, \mathcal{A'} \setminus \{D(y)\}) \vdash P$ if $U=add$, then the update $\mathcal{K} \to \mathcal{K'}$ is non-redundant.
\end{definition}

In Example~\ref{ex:update}, the updates $\mathcal{K} \to \mathcal K'_1$ and $\mathcal K \to \mathcal K'_2$ are non-redundant, whereas the update $\mathcal{K} \to \mathcal K'_3$ is redundant.

When creating a CF with respect to $C(x)$ for an individual $x$, it might not be desired to affect concept assertions for other individuals $y \neq x$. 
As a result, CF creation should restrict the allowed changes to those concepts $D$ such that $D(x)\in \mathcal A$. 
This way, CEs answer the question of what features of $x$ should be modified to change its classification.
\begin{definition}[Local counterfactual candidates]
Let $\mathcal K$ and $\mathcal K'$ be two KBs and $P$ be a CF request. 
If $\not \exists D(y) \in \Updelta(\mathcal A, \mathcal A')$  where $y \neq x$, then $\mathcal{K'}$ is a local counterfactual candidate for $P$ in $\mathcal K$.
\end{definition}

As proposed by Dandl et al.~\cite{Dandl2020Multi-Objective}, a CF is a counterfactual candidate with minimum feature changes. 
In a KB $\mathcal K$, we operationalize this via the notion of edit distance, i.e., the number of additions or removals of axioms necessary to create the CF $\mathcal K'$  from $\mathcal K$ such that $\mathcal K'$ fulfils a request $ P$.
We will only allow to update an ABox and define the edit distance with respect to changes in the ABox.

Let $\mathcal K=(\mathcal T, \mathcal A)$ and $\mathcal K'=(\mathcal T, \mathcal A')$ be two KBs.
We let $A_{\mathcal K}^{x}$ (resp., $A_{\mathcal K'}^{x}$) denote the set of all concept ($D$) and role names ($r$) in $\mathcal K$ ($\mathcal K'$), such that $D(x)\in \mathcal A$ or $r(x,y)\in \mathcal A$ (resp., in $\mathcal A'$). 
Then the counterfactuals for $P=\langle C(x),U\rangle $ in $\mathcal K$ are the KBs $\mathcal K'$ such that $\mathcal K'\vdash P$ and $\mathcal K'$ have the lowest \emph{edit distance} to $\mathcal K$.
The edit distance $\delta_{\mathit{cf}}(\mathcal {K},\mathcal {K'})$ of a CF is formalized via the symmetric difference $\Updelta(A_\mathcal K^x, A_\mathcal {K'}^x)$ as follows.
\begin{align}
\begin{split}
& \delta_{\mathit{cf}}(\mathcal K,\mathcal K') = |\Updelta(A_\mathcal K^x, A_\mathcal {K'}^x)| 
\label{eq:edit}
\end{split}
\end{align}
\begin{definition}[Minimal Changes]
Let $\mathcal K$ and $\mathcal K'$ be two KBs and $P$ be a CF request.
Then, $\mathcal K'$ is a \emph{counterfactual} for $P$ if $\mathcal K'$ is a counterfactual candidate for $P$ in $\mathcal K$ and $\delta_{\mathit{cf}}(\mathcal {K},\mathcal {K'})\leq \delta_{\mathit{cf}}(\mathcal {K},\mathcal {K''})$ for every such candidate $\mathcal K''$.
\end{definition}
Note that the 2nd (closeness of counterfactual to original instance) and the 3rd (sparse feature changes) objectives as specified by Dandl et al.~\cite{Dandl2020Multi-Objective} collapse in this criterion of minimal changes, since in $\mathcal{ELH}$, all features are discrete (a concept either holds or doesn't).

Next, we introduce two measures ($\mathit{l}_{\mathit{min}}$ and $\mathit{l}_{\mathit{mean}}$) of the likeliness (Dandl et al.'s objective 4~\cite{Dandl2020Multi-Objective}) with respect to an individual $x$ among different CFs fulfilling a removal request. 
Let $I_n$ denote the set of all \emph{existing negative} individuals ${y} \in N_I$ w.r.t. $C$, i.e., the set of individuals ${y}$ such that $\mathcal{K} \not\models C({y})$. 

The min-likeliness $l_{\mathrm{min}}$ is the minimal edit distance between $x$ and another individual $y$ such that $\mathcal K' \not \models C(y)$; 
the mean-likeliness $l_{\mathrm{mean}}$
denotes the \emph{average} edit distance between $x$ and all such individuals $y$.
Let $P=\langle C(x),rem\rangle$, $\mathcal K$ be a KB and $\mathcal K'$ be such that $\mathcal K'\vdash P$.
Moreover, let $y\in I_n$ 
and $\delta_{lm}(x,y,\mathcal K'):=|\Updelta (A_{\mathcal K'}^x,A_{\mathcal K'}^y)|$.
. 
Then,  $l_{\mathrm{min}}$ and $l_{\mathrm{mean}}$ are defined as follows:

\input{likeliness}
\subsection{Generation of Counterfactuals}
\label{subsec:gen}
\input{create_candidates_n_h}
\input{find_candidates}
\input{create_candidates_h}

Note that we distinguish between two types of counterfactual requests depending on whether a concept should hold ($U =add$) or should not hold ($U= rem$).
In order to fulfill a request $P=\langle C(x),rem\rangle$ for an assertion $C(x)$, 
we present Algorithm~\ref{alg:create-candidates_not_hold}. 
It is worth mentioning that axiom pinpointing for $\mathcal{EL}$~\cite{BaaderPS07,BaaderPS07a, BaaderP10} computes all justifications for the entailment of a subsumption ($X\sqsubseteq Y$) and not an assertion ($Y(x)$).
Whereas, we are not interested in finding every minimal subset of the TBox entailing subsumption but rather in determining whether subsumption holds for specific concepts in the ABox. 
Algorithm~\ref{alg:create-candidates_not_hold} applies a direct approach to compute all updates of an ABox .

\paragraph{Algorithm~\ref{alg:create-candidates_not_hold}.} 
Assume that $C=C_1\sqcap C_2\ldots \sqcap C_n$ and no $C_j$ contains an intersection on the outer level anymore (i.e., $C_j \neq A \sqcap B$ for any $C_j$).
Then Algorithm~\ref{alg:create-candidates_not_hold} begins by
splitting $C$ into the subconcepts $C_j$ (Line~$6$) in order to apply the subroutine from Algorithm~\ref{alg:find_candidates} to each $C_j$.
The sets of concepts added to the global variable \emph{candidates} by $\findcandidates(\mathcal K, \{C_{j}\}, \emptyset, \emptyset, x)$ are the desired concepts for removal to create one candidate from each set.
In other words, if $\mathcal K'$ is obtained from $\mathcal K$ by removing assertions as described in Lines~10--18, then $\mathcal K'\not\models C(x)$.
This is due to the reason that the subroutine $\findcandidates(\mathcal K, \{C_{j}\}, \emptyset, \emptyset, x)$ begins by finding all concepts subsumed by $C_j$ (the set $\mathcal E$) as well as all those for which the ABox contains an assertion (the set $\mathcal D$). 
In particular, $\mathcal D$ includes all the concepts that are essential while inferring $\mathcal K\models C_j(x)$.
As a result, each atomic and existential restriction in 
$\mathcal D$ is directly added to the set of relevant concepts since removing the ABox assertions corresponding to each of these concepts can fail the inference $\mathcal K\models C_j(x)$. 
In the case of existential restrictions $\exists r.A$, it does not matter if A is an atomic concept, another existential restriction, or an intersection, since to get local counterfactuals we simply remove the assertions $r(x,a)$ where $\mathcal{K} \models A(a)$ and do not need to remove $A(x)$.
Moreover, for each $E_j\in \mathcal E$ such that $E_j=E_{j_1}\sqcap\ldots\sqcap E_{j_n}$, removing an assertion for only one concept from $E_j$ can result in failing the inference $\mathcal K\models C_j(x)$.
The choice of allowing only one concept removal from $E_j$ is essential to avoid redundancy since $\mathcal K \not \models E_{j}(x)$ iff $\mathcal K \not \models E_{j_i}(x)$ for at least one $E_{j_i}\in E_j$.
The algorithm implements this by considering all possible combinations for different $j$ but taking only one concept from each $E_j$ (Line~$12$).

Moreover, the set of visited concepts is updated to avoid rechecking for concepts that are already visited (Lines~13--15).
Finally, the recursive calls to $\findcandidates(\mathcal K, P, visited, relevant, x)$ for each combination $P$ add further concepts as being relevant and, once completed, all the relevant concepts are returned as a candidate set.
This recursion for each $P$ is essential since $P$ also includes (removal) candidates for $C_j$.
Finally, Algorithm~\ref{alg:create-candidates_not_hold} uses the global variable \emph{candidates} and generates local counterfactuals corresponding to each set $\mathcal X$ of candidates.
Once all the updated KBs have been enumerated, the algorithm sorts the KBs with the least edit distance as counterfactuals according to Eq~\ref{eq:edit}. 
Furthermore, among these counterfactuals, the ones with the highest likeliness are selected for creating natural language CEs following the two likeliness measurements (Eqs~\ref{eq:likeliness1}-\ref{eq:likeliness2}).

\begin{theorem}\label{thm:neg}
The Algorithm~\ref{alg:create-candidates_not_hold} is sound and complete: Given a KB $\mathcal{K}$ and an update request $P=\langle C(x),\text{rem}\rangle$, such that $\mathcal K\not\vdash P$, Algorithm~\ref{alg:create-candidates_not_hold} returns a collection KBs such that  $\mathcal K'\in KBs$ if and only if $\mathcal K'$ is a local counterfactual for $P$ without redundancy in $\mathcal K$.  
\end{theorem}

\begin{proof}[Sketch]
The soundness ($\mathcal K' \vdash P$) follows since each KB $\mathcal K'$ is obtained from $\mathcal K$ after applying changes due to some candidate set $\mathcal X$.
Moreover, each $\mathcal X$ is returned due to the call $\findcandidates(\mathcal K, \{C_j\},\emptyset,\emptyset,x)$ for some concept $C_j\in C$.
This implies that for each $\mathcal K'$, there is some $C_j\in C$ such that $\mathcal K'\not\models C_j(x)$.
For completeness, notice that Lines~10--20 include the necessary and sufficient removal of assertions to create a KB $\mathcal K'$ that fulfills $P$. Therefore, every non-redundant local counterfactual $\mathcal K'$ is returned by the algorithm.
The full proof can be found in the repository~\footref{repo}. \qedhere

\end{proof}

Updating $\mathcal K$ to fulfill a request $\langle C(x), add\rangle$ is relatively easier to implement, since it suffices to add axioms implying $\mathcal K\models C(x)$.
This is similar to abduction since for a KB $\mathcal K$, a hypotheses set $H$ of assertions such that $\mathcal K\cup H\models C(x)$, yields a CF for $C(x)$ in $\mathcal K$. However, we differ from abduction in that we also allow to add fresh individuals. Algorithm~\ref{alg:create-candidates_hold} presents our approach .

\paragraph{Algorithm~\ref{alg:create-candidates_hold}}
The algorithm simply adds assertions $C_j(x)$ to the KB for each $C_j\in C$.
This yields the only non-redundant way to infer $\mathcal K' \models C(x)$.
Notice that if $C_j=\exists r.Z$ for some $j$ and a concept $Z$, then there are multiple options regarding the object $y$ such that the assertion $r(x,y)$ has to be added.
Either a new individual $y$ is added to $\mathcal K'$ and the process is repeated for $y$ regarding $Z$, or $y$ is an existing individual in $\mathcal K$ which is already in $Z$. If there exists no $y$ in $\mathcal K$ with $\mathcal K \models Z(y)$, of course it is not possible to create a local counterfactual.
In our case, we allow the addition of a fresh individual $y$. %
This achieves the desired goal since our main focus is on finding CEs for $C(x)$.
In practice, the decision to allow fresh individuals may depend on the application context.
Moreover, adding an assertion $D(x)$ to $\mathcal K'$ is redundant if $D$ is a subconcept of $C$ and a necessary consequence if $D$ is a superconcept of $C$.

\begin{theorem}\label{thm:pos}
The Algorithm~\ref{alg:create-candidates_hold} is sound and complete: Given a KB $\mathcal{K}$ and a request $P=\langle C(x),\text{add}\rangle$ such that $\mathcal K\not\vdash P$, Algorithm~Algorithm~\ref{alg:create-candidates_hold} returns an updated KB $\mathcal K'$ such that $\mathcal K'$ is a counterfactual for $P$ without redundancy in $\mathcal K$.  
\end{theorem}

Notice that, an obvious question emerges when considering inferred knowledge: how to treat the implicit assertions in an Abox for the selection of CEs? 

In other words, should the implicit knowledge also be considered when measuring the edit distance?
This leads one to chose that $\mathcal{A}$ should be fully materialized.
One rationale behind materialization  is: 1) all ``features'' (as authors in \cite{Dandl2020Multi-Objective} called it) that have to be removed or added to generate the counterfactual, even if implicit, should be counted in finding the counterfactual with least ``feature changes''. 2) Implicit features that do not have to be removed to create the counterfactual should not be removed.
For example, assume $\mathcal{T} \models A \sqsubseteq B \sqsubseteq C \sqsubseteq D$, $\mathcal{A} \models A(x)$, $P = \langle{C(x), rem}\rangle$, then $A(x), B(x), C(x)$ have to be removed and counted to the edit distance, but $D(x)$ does not. Materializing the KB before applying our algorithm achieves that result.

Nevertheless, this decision is optional and may depend on specific applications rather than being obligatory when utilizing our algorithm.

\section{Explanations Preferred by Users}
\label{sec:evaluation-method}

We conducted a survey in which we let participants rate different potential CEs against each other. We then compared the participants' preferences with the decisions made by our approach. Our hypothesis was that CEs generated by our algorithm will get positive ratings by study participants.
We used modified versions of the \emph{Family} and \emph{Animals} ontologies~\cite{Westphal2019SML-Bench} to evaluate our approach and materialized the ABoxes. These ontologies were chosen because the concept, role and individual names therein are familiar and understandable to average lay users---in contrast to, for example, ontologies related to bio-medicine or chemistry. We used the DL concept learner~\cite{Buhmann2016DL-Learner} with ELTL---the $\mathcal{EL}$ Tree Learner~\cite{Buhmann2014Inductive}---to learn the concepts to be used for counterfactual generation, since a future goal is to combine these programs to reach a fully automated explainable AI.

\subsection{Data Generation}

\input{table-datasets}

Table~\ref{tab:datasets} gives an overview of the datasets used for our user survey.
We used the DL concept learner~\cite{Buhmann2016DL-Learner} with ELTL to learn concepts from the Family and Animals ontologies that describe classes of family members or species of animals, respectively. A detailed description of the datasets and how we applied ELTL to find concepts for explaining can be found in the repository~\footref{repo}.
For the study, we created a CE from each counterfactual candidate that could be drawn from the learned concepts, the user ratings of the CEs selected by our algorithm with user ratings of other possible explanations.
To keep it consistent for the participants, we also presented 
explanations for the Family concepts for \texttt{Brother} and \texttt{Grandmother} as the corresponding concepts to \texttt{Sister} and \texttt{Grandfather} in the survey, even if ELTL did not correctly recognize these concepts and therefore our algorithm was not applied here.

\subsection{Setup of User Survey}

Using the generated concepts from the Family and Animals ontologies (see above) and their respective counterfactual candidates, we generated short stories of artificial intelligences classifying people in a family tree or animals and created a CE from each counterfactual candidate. We conducted an online survey via SoSciSurvey
in German. 
Participants were recruited through social networks and snowballing. 
At first, participants were informed about the content and goal of the survey and what CEs are, and later, the CEs were presented.

First, a scenario was described in which an AI would classify instances of family members or animals. Then, on every page, a classification made by an AI was presented in one sentence, followed by one or multiple sentences giving CEs for the classification, e.g. ``I would not have classified this animal as a turtle, if it did not have scales''. Within the two scenarios, classifications were presented in randomized order, one on each page. For the Family ontology, where each concept had led to two counterfactual candidates, the participants were randomly shown only one of the CEs. Because many explanations were quite similar (e.g. all concepts included counterfactual candidates referring to gender) it was made sure that they were presented mixed explanation types. Each explanation was accompanied by one item asking to rate on a scale from one to seven how helpful they perceived the explanation for understanding the decision of the program. For the animals scenario, participants were shown all counterfactual candidates (between two and five) at the same time, in random order, and presented the same rating scale for each of the explanations.

\subsection{Results of User Survey}

In the following, we present the results of our evaluation of our CE algorithm through a user survey.

\paragraph{Sample}
72 people took part in the survey. Age ranged between 20 and 69 (mean = 34.9, median = 32, standard deviation = 12.1, missing age data for one participant). 30 participants were female, 39 male and 3 diverse. Participants had mixed professions including both academic and non-academic ones, technical and non-technical.

\subsubsection{User ratings}
We used Wilcoxon signed-rank tests to calculate significance of deviation from the central value (4 on a scale from 1-7) of the rating of helpfulness for understanding for each CE. Tests were chosen given the fact that we use ad-hoc generated items, so we assume the ratings to be ordinal.
For all six concepts, users preferred the explanation that featured a role referring to relatives of the individual (\texttt{hasChild}, \texttt{hasSibling}) against the explanation that featured a $C_{i}$ referring to the individuals \texttt{gender} (all p$<$.001, except \texttt{Sister}: p$<$.01). These explanations (and only these) differed significantly from the central value (all p$<$.001).
For the Animaly ontology, we used Wilcoxon signed-rank tests as above (but for matched samples). For concepts with more than two counterfactuals, we used the Friedmann test. If present, participants always rated the CE mentioning an animal laying eggs as helpful (p$<$.01). Apart from that, decisions on this ontology showed no clear pattern, choosing (with p$<$.01) different features for different animals for explanation.

\input{table_evaluation}

\subsubsection{Comparison of algorithmic decisions with user ratings}
The match of our algorithms' decision with the participants' ratings can be seen in tablet~\ref{tab:eval}. We counted all CEs where participant ratings were significantly positive, so both algorithm and participants sometimes selected more than one explanation per scenario. Details on explanations selected by algorithm and participants can be found in our repository~\footref{repo}.

\subsubsection{Interpretation of results}
This study compared different CEs, i.e. explanations of the type ``what features of $x$ needed to be different for $x$ not to be a $C$?'' from user perspective. Overall, we suspect that participants preferred explanations referring to features that are rather unlikely and therefore more characteristic of the person or animal, though more studies would be needed here. In contrast, features which are very common (like being of a certain gender or having legs) were chosen less. This fits our likeliness measurement idea, since removing a feature that does not appear often in the population for candidate generation should also result in a rather high likeliness score using our definitions. Our algorithm partly manages to cover that, but could be improved. 

\section{Discussion}
\label{sec:discussion}


We showed that our approach can generate counterfactuals with minimal \emph{edit distance} measured by axiom additions and removals i.e. few features changes to the individual. Moreover, as there can be multiple counterfactuals per individual with minimal edit distance, we explored two \emph{likeliness} measures to choose among them.

Regarding the choice of the best counterfactual for an explanation, the results of our evaluation survey show room for improvement. In future work, a ``learning to rank'' algorithm might be used to automatically learn a likeliness measure.
Moreover,  $\mathcal{ELH}$ is a rather restricted DL. We plan to expand our algorithm to the more complex DL $\mathcal{ALC}$. The main challenges for $\mathcal{ALC}$ include negation and disjunctions. 

In $\mathcal{ALC}$, counterfactuals for $\langle C(x), add\rangle $ would differ further from abduction, since the ABox abduction does not return any explanation if $\mathcal K\models \neg C(x)$ but there might exist CFs obtained by removing the entailment of $\neg C(x)$ from $\mathcal K$.

\paragraph{Structure of ontologies}
A point open for discussion addresses how to deal with different structures of ontologies. 
As it is the case in the Family ontology, sometimes individuals might or might not have the same roles, while it is very unlikely for them to have the same objects (e.g. children) for these roles, which is affecting the likeliness measurement. Furthermore, it might depend on the ontology if sub-features of the changed feature should be counted into the edit distance, as we did. Thus, the usefulness of different distance calculation possibilities may depend on use cases and ontology structures.
\paragraph{Actionability}
While we tried to check for plausibility of counterfactual instances, our scoring mechanism cannot make sure yet that the axioms that were changed can actually be changed in the real world. One argument for CEs is that in many applications it might be interesting for data subjects to get to know how they can change their classification~\cite{Wachter2018Counterfactual}. However, the Family ontology shows an example of cases where this is not possible, since people cannot usually change their gender or relatives. Poyiadzi et al.~\cite{Poyiadzi2020FACE} discuss the relevance of actionability of counterfactuals. Our future work will put more focus on applications where actionability can be reached and how to do this.

\paragraph{Applications}
The Family ontology contains data about people's features and relations, as in DBpedia, YAGO, and Wikidata. CEs might be especially useful when humans are the classification target and demand a transparent justification of the result. To enrich KBs, additional information can be extracted from the web. Concept learning allows checking the consistency of the extracted information and inferring new (implicit) axioms from the explicitly stated axioms. Concept learning has been effectively applied to medical ontologies~\cite{Lehmann2010Concept}, but the learned concepts can become very long \cite{Kouagou2021Learning}, making them hard to grasp even for experts. Counterfactuals, which might even be verbalized in natural language, help to steer focus to the most important parts of the concept. Ultimately, we want to develop a chatbot that, in the spirit of XAI, provides users with natural language explanations of automatically learned concepts and can be applied to various use cases in areas including web science, medicine and finance. 

\enlargethispage{2\baselineskip}

\section{Conclusion}
We propose the first approach for generating CEs for concept assertions in the DL $\mathcal{ELH}$. Our approach performs well on the objective to generate counterfactual candidates which are similar to the individual. We discussed possibilities to improve the likeliness measurement of counterfactuals in accordance with findings from a user study. Our future work will move on to more complex DLs.

\bibliographystyle{ACM-Reference-Format}
\bibliography{bib}

\end{document}

%% file: figure_count_input.tex
\label{fig:ex}
\begin{figure}[t!]
\centering
\includegraphics[width=0.48\textwidth]{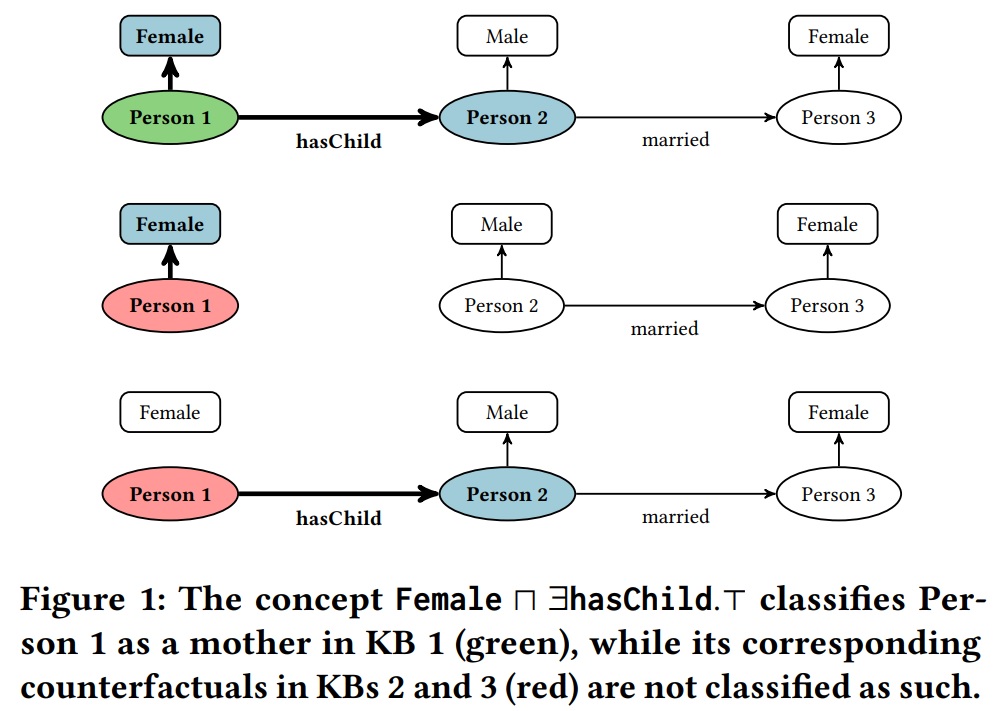}
\end{figure}

%% file: table-elh.tex
\newcommand*\rot{\rotatebox{90}}

\begin{table}[htb]
  \centering
  \caption{$\mathcal{ELH}$ description logic constructs. }
  \label{tab:elh}
  \setlength{\tabcolsep}{3.5pt}
  \begin{tabular}{@{}llll@{}}
  \toprule
  & \textbf{Syntax} & \textbf{Semantics} & \textbf{Construct} \\
  \midrule
  
  & $\top$        & $\Delta^\mathcal{I}$                                        & top concept \\
  & $A$        & $A^\mathcal{I} \subseteq \Delta^\mathcal{I}$  & atomic concepts \\ 
  &  $r$           & $r^\mathcal{I} \subseteq \Delta^\mathcal{I} \times \Delta^\mathcal{I}$        & atomic roles \\ \cmidrule{2-4}
  
  & $C \sqcap D$  & $C^\mathcal{I} \cap D^\mathcal{I}$                                    & intersection \\
  & $\exists r.C$ & $\{x \in \Delta^\mathcal{I} | \exists y \in \Delta^\mathcal{I}$       & existential restriction \\
  \rot{\rlap{\hspace{16pt}Concepts}}
  && $\mbox{ with } (x, y) \in r^\mathcal{I} \wedge y \in C^\mathcal{I}\}$ & \\
  \midrule
  
  & $ C(x)$  & $ x^\mathcal{I} \in C^\mathcal{I}$                                         & concept assertion \\
  \rot{\rlap{\hspace{-2pt}ABox}}
  & $ r(x, y) $  & $(x^\mathcal{I}, y^\mathcal{I}) \in r^\mathcal{I}$                               & role assertion \\\midrule

  & $C \sqsubseteq D$  & $C^\mathcal{I} \subseteq D^\mathcal{I}$      & concept subsumption \\
  \rot{\rlap{\hspace{-2pt}TBox}} &
  $r \sqsubseteq s$  & $r^I \subseteq s^\mathcal{I}$  & role subsumption \\  
  \bottomrule
  \end{tabular}
\end{table}

%% file: likeliness.tex
\begin{equation}
l_{\mathrm{min}}(P,\mathcal K') := \min_{y\in I_n} \delta_{lm}(x,y,\mathcal K')
\label{eq:likeliness1}
\end{equation}
\begin{equation}
l_{\mathrm{mean}}(P,\mathcal K') := \frac{1}{|I_n|} \sum_{{y\in I_n}} \delta_{lm}(x,y,\mathcal K')
\label{eq:likeliness2}
\end{equation}
%

%% file: create_candidates_n_h.tex
\begin{algorithm}[t]
	\DontPrintSemicolon
	\caption{Finds all non-redundant updates $\mathcal{K} \to \mathcal K'$ s.t. $\mathcal K' \not \models C(x)$}%
	\label{alg:create-candidates_not_hold}
    \textbf{Input:} KB $\mathcal{K}$ and a CF request $P= \langle C(x),rem \rangle $
    s.t. $\mathcal{K} \not \vdash P$\\
    	\textbf{Output:} Local counterfactuals $\mathcal K'$ in $\mathcal{K}$ sorted by likeliness such that with $\mathcal K' \vdash P$, global variable $candidates$ \\
        \SetKwFunction{createcandidates}{create\_candidates\_neg}
        \SetKwFunction{findcandidates}{find\_candidates}
        \SetKwFunction{acopy}{copy}
    	\SetKwProg{Pn}{Function}{:}{\KwRet $\mathit{cfs\_min},\mathit{cfs\_mean}$}
    	\Pn{\createcandidates{$\mathcal{K}$, $C$, x}}{%
    	    $KBs \gets \emptyset$\;
    	    $candidates \gets \emptyset$\;
         ${C} \gets \{C_{1}, \ldots,C_{n}\}$ \;
      		\For{$C_{j}$ in ${C}$}{%
      $\findcandidates(\mathcal{K}, \{C_{j}\}, \emptyset, \emptyset, x)$\;
      	}
       
      	\For{$\mathcal{X}$ in candidates}{
      	$\mathcal K_{\mathcal{X}}$ $\gets$ $\acopy(\mathcal{K})$\\
      	\For{$X \in \mathcal{X}$}{
      	\uIf{$X \equiv \top$}{%
      			$\mathcal{K}_{\mathcal{X}} \gets None$\; 
      		} 
      	\uElseIf{$X$ is an atomic class}{%
      	Remove $X(x)$ from $\mathcal K_{\mathcal{X}}$
      		}
      	\uElseIf{$X = \exists r.A$}{%
      	Remove all $\{r(x, a) \, | \, \mathcal{K} \models A(a)$\} from $\mathcal K_{\mathcal{X}}$\;
      	}
      	}
            $KBs \gets KBs \cup \{\mathcal K_{\mathcal{X}}\}$
       }
        $\mathit{cfs} \gets$ $\argmin \limits_{\mathcal K' \in KBs}$ $\delta_{\mathit{cf}}(\mathcal K,\mathcal K')$\;
        $\mathit{cfs\_min} \gets$ sort $\mathit{cfs}$ by $l_{\mathrm{min}}$\;
        $\mathit{cfs\_mean} \gets$ sort $\mathit{cfs}$ by $l_{\mathrm{mean}}$\;
      	}
\end{algorithm}

%% file: find_candidates.tex
\begin{algorithm}[t]
	\DontPrintSemicolon
	\caption{Recursively follow path of subsumptions, add set(s) of relevant concepts to global variable $candidates$}%
	\label{alg:find_candidates}
    \textbf{Input:} KB $\mathcal{K}=(\mathcal{T}, \mathcal{A})$, set of concepts to look at in this iteration ($c\_set$), set of concepts already visited, set of visited concepts relevant for this candidate, target individual $x$, global variable $candidates$ \\
        \SetKwFunction{findcandidates}{find\_candidates}
    	\SetKwProg{Pn}{Function}{:}{}
    	\Pn{\findcandidates{$\mathcal{K}$, c\_set, visited, relevant, $x$}}{%
            $relevant \gets relevant \cup c\_set$\;
            $\mathcal{E} \gets \{E | \exists F: (E \sqsubseteq F) \in \mathcal{T}, \mathcal{K} \models E \sqsubseteq C_{i}$ for some $C_{i} \in c\_set$\}\\
      	$\mathcal{D} \gets \{D| D \in \mathcal {E}, D \notin visited$ and $ D(x)\in \mathcal{A}$ \}\;
      	\For{each $D_j\in \mathcal D $
       }{
                $relevant \gets relevant \cup \{D_j\}$\;
      	}
      	\For{each $E_j=E_{j_1}\sqcap E_{j_2} \sqcap \ldots \sqcap E_{j_n} \in \mathcal E$}{
            $EI_{j} \gets \{E_{j_1},\ldots, E_{j_n}\}$
        }
            $\mathcal{P} \gets \{P \subseteq \cup_{j} EI_j \enskip | \enskip \forall_j |P \cap EI_j| \geq 1$ and $\not\exists p \in P \forall_j |P \setminus \{p\} \cap EI_j| \geq 1\}$\;
            \tcp*[l]{all non-redundant combinations covering each $EI_j$ 
            }
  		\For{each $E \in \mathcal{E}$}{
        $visited \gets visited \cup \{E\}$\;
      	}
      	\uIf{$\mathcal P = \emptyset$}{
            $candidates \gets candidates \cup \{relevant\}$
            \tcp*[l]{candidates is a set of sets}
      	}
      
        \Else{
            \For{each $P\in \mathcal P$}{
            \findcandidates{$\mathcal{K}$, P, visited, relevant, x}}
        }
    } 
\end{algorithm}

%% file: create_candidates_h.tex
\begin{algorithm}[t]
    \DontPrintSemicolon
	\caption{Updated KB $\mathcal{K}'$ such that $\mathcal{K}' \models C(x)$}%
	\label{alg:create-candidates_hold}
    \textbf{Input:} KB $\mathcal{K}$ and a CF request $P=\langle C(x),add\rangle $ s.t. $\mathcal K \not\vdash P$\\
    	\textbf{Output:} Local counterfactual $\mathcal{K}'$ in $\mathcal{K}$ such that $\mathcal {K}' \vdash P$\\
        \SetKwFunction{makehold}{create\_candidates\_pos}
        \SetKwFunction{acopy}{copy}
    	\SetKwProg{Pn}{Function}{:}{\KwRet {$\mathcal{K}'$} }
    	\Pn{\makehold{$\mathcal{K}$, $C$, $x$}}{%
            $\mathcal K'$ $\gets$ $\acopy(\mathcal{K})$\\
            ${C} \gets \{C_{1}, \ldots,C_{n}\}$ \;
      	\For{$C_{j}$ in ${C}$}{%
                \uIf{$C_j$ is an existential restriction $\exists r.D$}{
                    add $r(x,y)$ to $\mathcal{K}'$ for a fresh individual $y$ \;
                    $\mathcal{K}' \gets$ \makehold{$\mathcal{K}'$, $D$, $y$} \tcp*[l]{recursive call}
                }
      	    \uElse{
                    add $C_j(x)$ to $\mathcal{K}'$  
                }
           }
      } 
\end{algorithm}



%% file: table-datasets.tex
\begin{table}[tb]
\centering
\caption{Overview of the final, modified datasets in terms of number of instances ($N_I$), axioms, atomic concepts and roles.
}
\label{tab:datasets}
\setlength{\tabcolsep}{9.5pt}
\begin{tabular}{@{}lcccc@{}}
  \toprule
   & \textbf{Instances} & \textbf{Axioms} & {\textbf{Atomic}} & \textbf{Roles} \\
  \textbf{Dataset} & & & {\textbf{Concepts}} & \textbf{}  \\
  \midrule
  Family  &           202 &           2,033 &  18 & 5 \\
  Animals & \phantom{0}28 & \phantom{0,}170 &  19 & 4 \\
  \bottomrule
\end{tabular}
\end{table}

%% file: table_evaluation.tex
\begin{table}[tb]
\centering
\caption{Overview of alignment of algorithm decisions with user ratings as ground truth. Notation: T = True, F = False, P = Positive, N = Negative
}
\label{tab:eval}
\setlength{\tabcolsep}{9.5pt}
\begin{tabular}{@{}cccccc@{}}
  \toprule
  $l_{\mathrm{min}}$   & \textbf{TP} & \textbf{TN} & {\textbf{FP}} & \textbf{FN}  & \textbf{F1-Score}\\
  \midrule
  Family  & 4 & 4 &  0 & 0 & 1.0\\
  Animals & 3 & 6 & 5 & 7 & 0.33\\
\end{tabular}

\begin{tabular}{@{}cccccc@{}}
  \toprule
   $l_{\mathrm{mean}}$ & \textbf{TP} & \textbf{TN} & {\textbf{FP}} & \textbf{FN}  & \textbf{F1-Score}\\
  \midrule
  Family  &4 & 4 &  0 & 0 & 1.0\\
  Animals & 3 & 8 & 3 & 7 & 0.37 \\
  \bottomrule
\end{tabular}
\end{table}

